\documentclass{article} 
\PassOptionsToPackage{table,dvipsnames}{xcolor}
\usepackage{iclr2027_conference,times}

\usepackage{amsmath,amsfonts,bm}

\def\eqref#1{equation~\ref{#1}}

\def\1{\bm{1}}

\DeclareMathAlphabet{\mathsfit}{\encodingdefault}{\sfdefault}{m}{sl}
\SetMathAlphabet{\mathsfit}{bold}{\encodingdefault}{\sfdefault}{bx}{n}

\usepackage{hyperref}
\usepackage{url}
\usepackage{booktabs}
\usepackage{amsmath}
\usepackage{amsfonts}
\usepackage{amssymb}
\usepackage{amsthm}
\usepackage{bm}
\usepackage{nicefrac}
\usepackage{microtype}
\usepackage[table,dvipsnames]{xcolor}
\usepackage{caption}
\usepackage{multirow}
\usepackage{array}
\usepackage{xspace}
\usepackage[noend]{algorithmic}
\usepackage{algorithm}
\usepackage{graphicx}
\usepackage{tikz}
\usepackage{placeins}
\usepackage{float}
\usetikzlibrary{arrows.meta,positioning,calc,shapes.symbols}
\usepackage{cleveref}
\crefname{appendixsection}{appendix}{appendices}
\Crefname{appendixsection}{Appendix}{Appendices}
\crefname{appendixfigure}{appendix figure}{appendix figures}
\Crefname{appendixfigure}{Appendix figure}{Appendix figures}
\crefname{appendixtable}{appendix table}{appendix tables}
\Crefname{appendixtable}{Appendix table}{Appendix tables}
\crefname{appendixequation}{appendix equation}{appendix equations}
\Crefname{appendixequation}{Appendix equation}{Appendix equations}
\newcommand{\UseAppendixReferenceNames}{%
  \crefalias{section}{appendixsection}%
  \crefalias{subsection}{appendixsection}%
  \crefalias{subsubsection}{appendixsection}%
  \crefalias{figure}{appendixfigure}%
  \crefalias{table}{appendixtable}%
  \crefalias{equation}{appendixequation}%
}

\graphicspath{{figures/}}

\hypersetup{
  colorlinks=true,
  linkcolor=MidnightBlue,
  citecolor=MidnightBlue,
  urlcolor=RoyalBlue
}
\newtheorem{lemma}{Lemma}
\newtheorem{proposition}{Proposition}

\usepackage{caption}
\usepackage{subcaption}
\usepackage{afterpage}
\usepackage[dvipsnames]{xcolor}
\usepackage{hyperref}

\usepackage[most]{tcolorbox}
\usepackage{xcolor}

\newtcolorbox{promptbox}[1]{
  enhanced,
  colback=white,
  colframe=black,
  boxrule=0.8pt,
  arc=6pt,
  outer arc=6pt,
  title=#1,
  colbacktitle=black,
  coltitle=white,
  fonttitle=\large\bfseries,
  left=10pt,
  right=10pt,
  top=10pt,
  bottom=8pt,
  toptitle=2pt,
  bottomtitle=2pt,
  lefttitle=10pt,
  righttitle=10pt,
}

\newcommand{\TreeHCA}{\textsc{TreeHCA}\xspace}
\newcommand{\HCA}{\mathrm{HCA}}
\newcommand{\ATPO}{AT$^{2}$PO\xspace}

\newcommand{\EE}{\mathbb{E}}
\newcommand{\PP}{\mathbb{P}}
\renewcommand{\1}{\mathbf{1}}
\newcommand{\TT}{\mathcal{T}}

\newcommand{\Children}{\mathrm{Ch}}
\newcommand{\pa}{\mathrm{pa}}

\definecolor{rankfirst}{HTML}{A0D8B5}
\definecolor{ranksecond}{HTML}{C9E8D5}
\definecolor{rankthird}{HTML}{E9F5ED}
\definecolor{tablegroup}{HTML}{F0F5F2}
\newcommand{\rkone}[1]{\cellcolor{rankfirst}\textcolor{black}{#1}}
\newcommand{\rktwo}[1]{\cellcolor{ranksecond}\textcolor{black}{#1}}
\newcommand{\rkthree}[1]{\cellcolor{rankthird}\textcolor{black}{#1}}

\title{\TreeHCA: Hindsight-Guided Exploration and Credit Assignment for Multi-Turn LLM Agents}

\author{%
  \small Debajoy Mukherjee\thanks{Correspondence to: Debajoy Mukherjee \texttt{<debajoym98@tamu.edu>}.} \quad
  Jeremy Carleton \quad
  Sushil Vemuri \quad
  Srinivas Shakkottai \quad
  Dileep Kalathil \\
  \small\normalfont Department of Electrical and Computer Engineering, Texas A\&M University, College Station, Texas, USA
}

\iclrfinalcopy
\begin{document}

\frenchspacing

\maketitle

\begin{abstract}
Outcome-supervised group-relative policy optimization (GRPO) gives every token in an LLM agent's trajectory the same advantage, leaving useful and unhelpful decisions undifferentiated. Hindsight credit assignment (HCA) can refine this feedback, but its estimates depend on trajectories containing informative alternatives at consequential histories. We introduce \TreeHCA, which couples \emph{sampling via hindsight ratios} to generate these alternatives with \emph{hindsight credit assignment} to evaluate their contribution. We use changes in reference-answer likelihood to revisit intermediate histories and select alternative continuations, forming a tree of reasoning and tool interactions. This exploration changes the distribution of sampled continuations. We derive the corresponding inverse-hindsight importance correction, including an exact finite-pool identity, and use its structure in a normalized tree backup that propagates terminal outcomes into turn-level advantages. This common hindsight structure is \TreeHCA's central design principle, linking informative exploration to hindsight-guided credit assignment. Across seven question-answering benchmarks and three Qwen backbones, \TreeHCA outperforms every baseline in aggregate exact match, validating our paired design for multi-turn LLM training. Our code is available at
\href{https://github.com/TheBlackCat22/TreeHCA}{github.com/TheBlackCat22/TreeHCA}.
\end{abstract}

\begin{figure}[!ht]
\centering
\includegraphics[width=\linewidth]{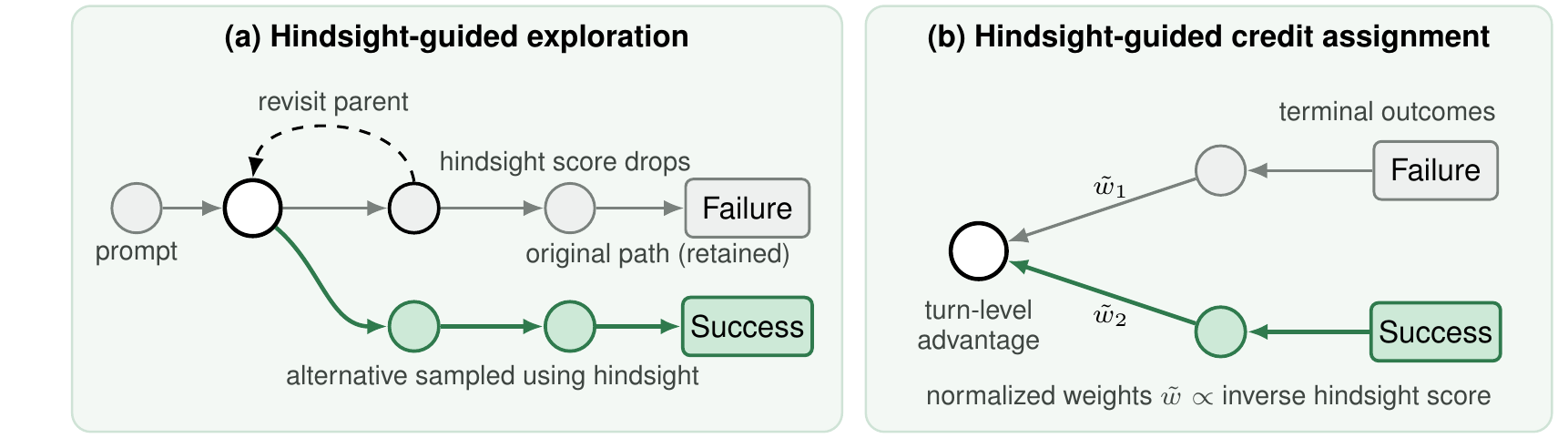}
\caption{\textbf{\TreeHCA pairs exploration with credit assignment.}
\textbf{(a)} A drop in hindsight score identifies a parent history to revisit.
The sampled alternative continues alongside the retained original path, here
illustrated with successful and unsuccessful outcomes.
\textbf{(b)} Outcome signals propagate backward using normalized inverse-score
weights to construct turn-level advantages.}
\label{fig:treehca_main}
\end{figure}

\section{Introduction}
\label{sec:introduction}

An LLM agent solves a multi-turn task through reasoning steps and tool calls. With terminal verification, success feedback arrives only after a rollout is completed. Outcome-supervised group-relative policy optimization (GRPO) converts this feedback into a trajectory advantage and \emph{assigns it to every generated token} \citep{GRPO}. A successful trajectory can reinforce unhelpful turns, while a failed one can penalize useful decisions. Therefore, a central challenge in providing feedback for training multi-turn LLM agents is to score intermediate decisions by their contribution to success.

\emph{How should we assign credit to intermediate decisions?} Hindsight credit assignment (HCA) relates an action's probability to the subsequent outcome \citep{harutyunyan2019hindsight}. An action that appears disproportionately often in successful continuations receives a larger success-conditioned hindsight ratio. This ratio has an exact change-of-measure interpretation, which connects outcome-conditioned observations to the action value under the original policy. In a language agent, the space of interaction histories is large, so the usefulness of this construction depends on obtaining informative alternatives from comparable histories. GiGPO forms groups of comparable states to obtain step-level relative advantages, and HCAPO uses outcome-conditioned LLM scores to refine trajectory feedback \citep{feng2025groupingrouppolicyoptimizationllm,tan2026hindsightcreditassignmentlonghorizon}. These methods make intermediate credit available from policy-generated chains. However, their estimates critically depend on these chains containing informative alternatives. In practice, such alternatives may be sparse at consequential histories, leaving the sampled trajectories insufficiently informative for turn-level credit assignment.

It is therefore crucially important to gather this information by revisiting histories from which alternative decisions can change the outcome. When an observed turn reduces the likelihood of success, we need to return to its parent and sample new continuations while retaining the original path. This produces comparisons under a shared prefix. Entropy-guided exploration allocates sampling effort to uncertain decisions \citep{AEPO,zong2026at2poagenticturnbasedpolicy}. However, uncertainty alone does not reveal a decision's contribution to success. We instead use hindsight to guide exploration toward successful outcomes. Hindsight scores identify turns worth revisiting and promising alternatives to sample from their parent histories. We therefore introduce \emph{hindsight-guided sampling} as an exploration strategy for obtaining the information needed to evaluate intermediate decisions.

We then make a complementary connection that is key to this paper: how we gather evidence through exploration determines how we must weight it when estimating turn-level advantages. By selecting continuations through hindsight, we change their distribution in the rollout tree. We must therefore account for that selection when estimating advantages for the original rollout policy. For a hindsight-tilted proposal, we derive the inverse hindsight ratio as the required importance weight. Our central insight is this correspondence: \emph{the hindsight ratio guides exploration toward informative continuations, and its inverse accounts for that sampling preference when estimating credit.}

We develop this correspondence in \textbf{\TreeHCA}. Its core idea has two parts. First, we use \emph{hindsight-guided sampling} to generate informative trajectories. Changes in the model's likelihood of the reference answer identify turns to revisit and guide our selection of alternative continuations. Second, we use \emph{credit assignment to samples} to accord advantages to intermediate turns according to the outcomes of their continuations. We propagate group-normalized terminal rewards through the tree using inverse hindsight weights to counterbalance the selection preference. The resulting turn-level signals reflect the successful and unsuccessful continuations supported by each intermediate decision. The key to \TreeHCA is that these operations follow the same hindsight change of measure: the score that directs exploration also determines the inverse correction used to compute credit. This pairing is the central design principle behind the gains evaluated in \Cref{sec:experiments}. See \Cref{fig:treehca_main}.

We make three contributions. \textbf{(1) Hindsight-guided exploration.} We formulate an exploration rule that uses turn-level hindsight scores to select histories and alternative continuations, connecting the exact hindsight ratio to candidate action values. \textbf{(2) Hindsight-guided credit assignment.} We derive the importance correction for hindsight-guided selection, including an exact finite-pool identity, and use its inverse-score structure to construct turn-level advantages through a normalized tree backup. \textbf{(3) Experimental validation of the paired construction.} We evaluate \TreeHCA on seven question-answering benchmarks with Qwen2.5-3B-Instruct, Qwen3-4B, and Qwen3-8B. Across the reported model and task settings, \TreeHCA outperforms all compared methods in aggregate exact match (EM). The paired design also achieves the strongest EM and F1 in the component study.

\section{Related Work}
\label{sec:related_work}

\paragraph{Intermediate credit from terminal outcomes.}
HCA expresses action values through outcome-conditioned action probabilities \citep{harutyunyan2019hindsight}. For LLM agents, GiGPO constructs groups around recurring or comparable environment states and combines step-level and trajectory-level advantages \citep{feng2025groupingrouppolicyoptimizationllm}. HCAPO uses the LLM as a hindsight critic, normalizes its outcome-conditioned action scores within a trajectory, and combines them with group-relative feedback \citep{tan2026hindsightcreditassignmentlonghorizon}. IGPO uses changes in reference-answer likelihood as information-gain rewards \citep{wang2026informationgainbasedpolicyoptimization}. These mechanisms extract intermediate learning signals from linear rollouts. \TreeHCA uses outcome relevance to select the continuations supplying these signals.

\paragraph{Tree sampling and turn-level learning.}Tree-GRPO shares prefixes in LLM-agent rollouts and constructs relative advantages within and across trees \citep{ji2025treesearch}. In diffusion models, TreeGRPO backs terminal advantages through a denoising tree \citep{ding2026treegrpo}. AEPO allocates sampling using entropy-based information proxies and shapes the optimization signal, while \ATPO combines entropy-guided exploration with an entropy-weighted tree backup \citep{AEPO,zong2026at2poagenticturnbasedpolicy}. IGRPO uses reference-answer likelihood to guide adaptive rollout expansion \citep{IGRPO}. These approaches motivate examining exploration and credit together. Our focus is the sampling distribution induced by hindsight selection and the corresponding inverse weighting in the backup. The separate policy-update ratio compares the updated and behavior policies. \Cref{app:extended_related_work} provides a detailed review of these sampling and credit mechanisms.

\section{Problem Formulation}
\label{sec:setup}

We consider an agent that starts from a question $x$ and alternates reasoning, tool use, and observation before producing an answer. At turn $t$, its history $h_{<t}$ contains the prompt and preceding interactions. The policy generates a reasoning sequence $s_t$ and an action $a_t$, which we treat jointly as a decision $u_t=(s_t,a_t)$. Executing a tool action produces an observation $o_t$. A completed turn is $n_t=(u_t,o_t)$, and a trajectory is $\tau=(x,n_0,\ldots,n_T)$. A final-answer action terminates the trajectory. The verifier supplies a binary terminal reward $R(\tau)\in\{0,1\}$, and the objective is
\begin{equation}
J(\theta)=\EE_{\tau\sim P_{\pi_\theta}}[R(\tau)],
\label{eq:rl_objective}
\end{equation}
where $P_{\pi_\theta}$ includes both policy decisions and environment responses.

During one rollout-collection round, we hold the behavior policy $\pi=\pi_{\theta_{\rm old}}$ fixed. All success probabilities and hindsight distributions below refer to continuations under this policy. Write $V^\pi(h)=\PP_\pi(R=1\mid h)$ and $Q^\pi(h,u)=\PP_\pi(R=1\mid h,U=u)$. The latter averages over the observation following $u$ and subsequent policy decisions. For a completed child $c=(u,o)$ of $n$,
\begin{equation}
p_\pi(c\mid n)=\pi(u\mid h(n))P_{\rm env}(o\mid h(n),u).
\label{eq:base_child_kernel}
\end{equation}
The history $h(n)$ includes the full prefix ending at $n$. Multiple trajectories with shared prefixes form a tree $\TT$. Its sibling nodes supply alternative completed turns from a common parent, and its leaves carry the verifier outcomes used for learning. We use the reference answer $y^\star$ during training to score these alternatives, and the learned agent uses its ordinary policy and search tool during evaluation.

\section{Hindsight-Guided Exploration}
\label{sec:hindsight_guided_branching}

We first develop \TreeHCA's hindsight exploration policy. This is the sampling part of the paired construction: we use scores indicative of final outcome likelihood to decide where to gather informative alternatives. A turn that reduces the likelihood of success identifies a parent history worth revisiting. We sample from the history before that decision, so an alternative can change both the next action and the evidence subsequently available to the agent. Retaining the original continuation provides a comparison under the same prefix. We formalize this criterion through the exact hindsight distribution, then construct its practical score and tree-sampling rule. The terminal verifier evaluates each retained path, supplying the outcomes whose credit we compute in \Cref{sec:credit_assignment}.

\subsection{Sampling via Credit Assignment}
\label{sec:branch_credit}

For a supported decision $u$ at history $h$ and an outcome $r$ of positive probability, the hindsight ratio is
\begin{equation}
\rho_{\HCA}(h,u;r)
=\frac{\PP_\pi(U=u\mid h,R=r)}{\pi(u\mid h)}
=\frac{\PP_\pi(R=r\mid h,U=u)}{\PP_\pi(R=r\mid h)}.
\label{eq:future_outcome_hca_ratio}
\end{equation}
Bayes' rule gives the second equality. A ratio above one means that $u$ is more prevalent among continuations with outcome $r$ than under the rollout policy. For success,
\begin{equation}
\rho_{\HCA}(h,u;1)=\frac{Q^\pi(h,u)}{V^\pi(h)},
\qquad
\log\rho_{\HCA}(h,u;1)=\log Q^\pi(h,u)-\log V^\pi(h).
\label{eq:success_hindsight_branching_ratio}
\end{equation}
Thus a negative log ratio identifies a decision after which success is less likely than it was at the parent history. Exploring from that parent supplies alternatives to the observed decision.

The same quantity ranks the alternatives. The following identity instantiates the return-conditioned HCA result of \citet{harutyunyan2019hindsight} for our turn-level process.
\begin{lemma}[Hindsight ranking of candidate decisions]
\label{lem:hca_counterfactual_actions}
For $\pi(u\mid h)>0$, let $T(h,\pi)$ be the distribution of complete continuations from $h$ under $\pi$ and the environment. Then
\begin{equation}
Q^\pi(h,u)=\EE_{\tau\sim T(h,\pi)}
\left[R(\tau)\rho_{\HCA}(h,u;R(\tau))\right].
\label{eq:hca_counterfactual_action_value}
\end{equation}
For binary terminal rewards and $V^\pi(h)>0$, this reduces to $Q^\pi(h,u)=V^\pi(h)\rho_{\HCA}(h,u;1)$.
\end{lemma}
The parent value is common to all candidates, so the exact hindsight ratio ranks them by expected terminal reward. A low ratio identifies an observed turn to revisit, while a high ratio identifies an attractive alternative from its parent (\cref{app:hca_ranking_proof}).

\subsection{Scoring Histories with the Policy LLM}
\label{sec:practical_score}

We obtain a training-time proxy using the reference answer $y^\star=(y_1^\star,\ldots,y_J^\star)$. Following the use of answer likelihood as an information signal \citep{wang2026informationgainbasedpolicyoptimization,IGRPO}, we append a fixed answer-eliciting instruction $s^\star$ to a history and teacher-force its answer tokens. The score is the mean token log probability
\begin{equation}
\ell_\pi(h;y^\star)=\frac{1}{J}\sum_{j=1}^{J}
\log\pi(y_j^\star\mid h,s^\star,y_{<j}^\star).
\label{eq:teacher_forced_score}
\end{equation}
For a completed turn $c=(u,o)$, let $h^-(c)$ and $h^+(c)$ be the histories before and after that turn.
\begin{equation}
\delta(c)=\log\widehat\rho_c
=\ell_\pi(h^+(c);y^\star)-\ell_\pi(h^-(c);y^\star).
\label{eq:practical_log_hindsight_ratio}
\end{equation}
This score measures how the realized turn changes compatibility with the reference answer. It includes the information supplied by the tool observation. A negative score identifies a reduction in that compatibility and motivates collecting alternatives from $h^-(c)$. The teacher-forced likelihood provides a computable surrogate for the conditional success probabilities in \Cref{eq:success_hindsight_branching_ratio}.

\begin{figure}[b]
\centering
\includegraphics[width=\linewidth]{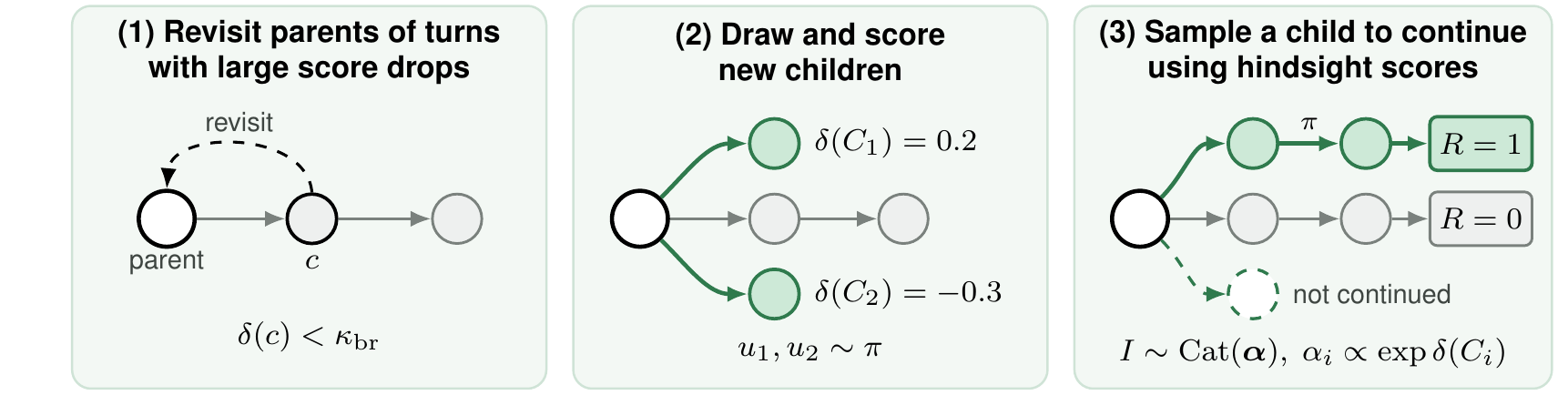}
\caption{Hindsight-guided exploration. \textbf{(1)} A low score on an observed completed turn identifies its parent for resampling. \textbf{(2)} From that starting point, $N=2$ policy decisions are executed and scored as completed candidates. \textbf{(3)} A categorical draw selects one using normalized hindsight scores. It continues under $\pi$ to a verifier reward, while the original child and subtree remain. Scores and outcomes are illustrative.}
\label{fig:hindsight_branching}
\end{figure}

\subsection{Selecting and Completing Alternatives}
\label{sec:hindsight_filtering}

Starting from $M$ complete policy rollouts, we identify turns with $\delta(c)<\kappa_{\rm br}$ for a fixed negative threshold. In each expansion round, we select up to $K$ lowest-scoring turns. If fewer than $K$ satisfy the threshold, uniform tie-breaking fills the allocation from the available turns. Each selected turn identifies its parent as a starting point for exploration, as illustrated in \Cref{fig:hindsight_branching}.

At a selected parent, we draw $N$ candidate decisions from $\pi$, execute their actions to obtain completed candidates $C_1,\ldots,C_N$, and score them. We sample one candidate index with
\begin{equation}
\alpha_i=\PP(I=i\mid C_{1:N})
=\frac{\exp\delta(C_i)}{\sum_{j=1}^N\exp\delta(C_j)}
=\frac{\widehat\rho_{C_i}}{\sum_{j=1}^N\widehat\rho_{C_j}}.
\label{eq:hindsight_candidate_sampling}
\end{equation}
Only $C_I$ is continued to termination, and the original child and its subtree remain available for learning. This spends $N$ turn executions and one full continuation per expansion. Stochastic selection allocates the continuation budget across scored alternatives.

\Cref{alg:tree_construction} gives the complete construction. The tree now contains both original rollouts and selectively retained alternatives. Their outcomes provide comparisons under a shared prefix, while their selection probabilities determine how these comparisons should contribute to a rollout-policy expectation. This is the estimation task addressed next.

\begin{algorithm}[t]
\caption{Hindsight-guided tree construction}
\label{alg:tree_construction}
\begin{algorithmic}[1]
\REQUIRE Prompt $x$, reference answer $y^\star$, frozen policy $\pi$, initial count $M$, rounds $L$, exploration budget $K$, candidates $N$, threshold $\kappa_{\rm br}<0$.
\STATE Initialize $\TT$ with $M$ complete policy rollouts from $x$.
\FOR{$\ell=1,\ldots,L$}
\STATE Score each non-root node by $b(c)=\min\{\delta(c),\kappa_{\rm br}\}$. \label{line:score}
\STATE Select up to $K$ nodes with the smallest scores, breaking ties uniformly.
\FOR{each selected node $c$}
\STATE Draw $N$ decisions from $\pi(\cdot\mid h^-(c))$ and execute them.
\STATE Score the completed candidates and sample $I$ using \Cref{eq:hindsight_candidate_sampling}.
\STATE Attach $C_I$ to $\pa(c)$ and complete it under $\pi$, retaining $c$ and its subtree.
\ENDFOR
\ENDFOR
\STATE \textbf{return} $\TT$ with its terminal verifier rewards.
\end{algorithmic}
\end{algorithm}

\section{Hindsight-Guided Credit Assignment}
\label{sec:credit_assignment}

We now develop the complementary part of the pair which includes assigning credit to the samples gathered by \TreeHCA's exploration policy. Hindsight-guided selection favors high-scoring continuations, while the learning signal must reflect their value under the original rollout policy. The same scores that guide exploration therefore determine the inverse weighting needed for estimation. We derive this correspondence for the specified sampling laws and use its inverse-score structure to construct turn-level advantages on the tree.

\subsection{The Sampling and Estimation Correspondence}
\label{sec:hindsight_corrected_tree_backup}

Fix a parent $n$ and write $p(c)=p_\pi(c\mid n)$ for its completed-child kernel. For a positive score $s(c)$ with $Z=\EE_p[s(C)]\in(0,\infty)$, consider the tilted proposal
\begin{equation}
q_s(c)=\frac{p(c)s(c)}{Z},
\qquad
\frac{p(c)}{q_s(c)}=\frac{Z}{s(c)}.
\label{eq:child_hindsight_ratio}
\end{equation}
The score specifies which children are favored in sampling, and its inverse supplies the change of measure back to the rollout policy. If $s(c)$ is the exact success-conditioned hindsight ratio of the completed child, this proposal is the success-conditioned child distribution on its support.

\begin{proposition}[Credit-guided sampling and inverse correction]
\label{prop:hindsight_correction}
Let $f$ be integrable under $p$. If $q_s$ covers the support of $pf$, then
\begin{equation}
\EE_{q_s}\!\left[\frac{Z}{s(C)}f(C)\right]=\EE_p[f(C)].
\label{eq:policy_proposal_importance_identity}
\end{equation}
For a finite pool $C_1,\ldots,C_N\stackrel{\rm iid}{\sim}p$ and selection probabilities $\alpha_i=s(C_i)/\sum_j s(C_j)$, the corresponding identity is
\begin{equation}
\EE\!\left[\frac{f(C_I)}{N\alpha_I}\,\middle|\,C_{1:N}\right]
=\frac{1}{N}\sum_{i=1}^N f(C_i).
\label{eq:finite_pool_correction}
\end{equation}
Consequently, averaging over the pool recovers $\EE_p[f(C)]$.
\end{proposition}

The finite pool also determines the selected child's marginal law. For $N$ candidates and a fixed first candidate $c$, symmetry gives
\begin{equation}
q_N(c)=N p(c)\,\EE\!\left[
\frac{s(c)}{s(c)+\sum_{j=2}^N s(C_j)}\right],
\label{eq:finite_pool_marginal_main}
\end{equation}
where the expectation is over the remaining independent policy candidates. At $N=1$, selection follows $p$. For scores bounded above and away from zero, concentration of the remaining scores around their mean gives $q_N(c)\to p(c)s(c)/Z$. Thus the direct tilt describes large-pool selection, while \Cref{eq:finite_pool_correction} provides an exact correction at every finite pool size.

\Cref{app:importance_proof} proves these identities by a change of measure and a conditional weighted sum, and gives a two-child illustration. For the finite-pool selection in \Cref{eq:hindsight_candidate_sampling}, the weight is the mean candidate score divided by the selected score. The identity holds for any strictly positive score. It also holds for a conditionally unbiased estimate of $f(C_I)$ after averaging over continuation randomness.

The direction of this weighting follows the target distribution. In \Cref{lem:hca_counterfactual_actions}, the hindsight ratio multiplies a return sampled under the rollout policy to recover an action-conditioned value. In \Cref{prop:hindsight_correction}, the sampling law itself favors children with large scores, so recovering a rollout-policy expectation uses the inverse score. For the exact hindsight score, these identities express the two directions of the conditional-distribution change of measure.

These identities explain why exploration and credit must be designed together. Hindsight-guided selection makes high-scoring continuations more common, so their observed frequency alone cannot represent the rollout policy. Inverse weighting compensates for this preference when computing a policy-based expectation. The direct-tilt identity gives the correction for score-proportional sampling, and the finite-pool identity gives it conditional on the sampled candidates. \TreeHCA carries the same inverse-score structure into a normalized backup over retained children, propagating their outcome signals to intermediate turns. \Cref{app:finite_pool,app:snis} provide the derivations.

\subsection{From Leaf Rewards to Turn-Level Advantages}
\label{sec:tree_backup}

For the leaves belonging to one prompt, let $\mu_R$ and $\sigma_R$ be the mean and standard deviation of the verifier rewards. Initialize a leaf $\ell$ with its group-relative signal
\begin{equation}
\widehat A^{\rm tree}(\ell)=
\begin{cases}
\bigl(R(\ell)-\mu_R\bigr)/\sigma_R,&\sigma_R>0,\\
0,&\sigma_R=0.
\end{cases}
\label{eq:leaf_advantage}
\end{equation}
The zero-variance case supplies a zero learning signal when every trajectory has the same outcome. The tree then propagates these signals toward the root. For the observed children $\Children(n)$, define
\begin{equation}
\widetilde w_{n,c}=
\frac{\widehat\rho_c^{-1}}{\sum_{c'\in\Children(n)}\widehat\rho_{c'}^{-1}},
\qquad
\widehat A^{\rm tree}(n)=
\sum_{c\in\Children(n)}\widetilde w_{n,c}\widehat A^{\rm tree}(c).
\label{eq:hindsight_corrected_tree_backup}
\end{equation}
The recursion proceeds in reverse topological order. A node with one child inherits that child's signal. At a parent with multiple children, the inverse scores counterbalance the preference toward high-scoring alternatives, and normalization produces a convex combination of the child signals.

Under independent sampling from the specified tilted proposal, self-normalized inverse-score weighting consistently estimates the base-policy expectation. The tree operator applies this normalization to a finite, adaptively constructed tree with proxy scores. Its finite-tree behavior is evaluated through the credit-rule ablation in \Cref{sec:component_results}. The exact identities in \Cref{prop:hindsight_correction} explain the correction's form, while the normalization preserves the scale and common shifts of the group-relative signal. The appendix states the estimator conditions and separates finite-pool selection from the retained siblings.

\subsection{Turn-Level Policy Optimization}
\label{sec:gspo_loss}

Each policy-generated token receives the advantage of its completed turn. If token $t$ of trajectory $y_i$ belongs to node $n_{i,\kappa_i(t)}$, set $A_{i,t}=\widehat A^{\rm tree}(n_{i,\kappa_i(t)})$. Tool outputs and prompt tokens remain in the conditioning history and are masked from the loss. Let $M_{i,t}$ be the policy-token mask and $N_i=\sum_t M_{i,t}$. We use the GSPO-token surrogate \citep{GSPO}, whose detached sequence ratio scales the token gradient:
\begin{align}
r_i(\theta)&=\exp\!\left[\frac{1}{N_i}\sum_t M_{i,t}
\log\frac{\pi_\theta(y_{i,t}\mid h_{i,t})}{\pi(y_{i,t}\mid h_{i,t})}\right],\label{eq:sequence_ratio}\\
s_{i,t}(\theta)&=\frac{\pi_\theta(y_{i,t}\mid h_{i,t})}
{\operatorname{sg}[\pi_\theta(y_{i,t}\mid h_{i,t})]}
\operatorname{sg}[r_i(\theta)].\label{eq:importance-ratio}
\end{align}
Here $\operatorname{sg}$ denotes stop-gradient. This policy-update ratio compares the updated and behavior policies. The inverse-hindsight weights in \Cref{eq:hindsight_corrected_tree_backup} instead counterbalance continuation selection.

Let $\mathcal B_\pi(\cdot\mid x)$ denote the tree-generation law of \Cref{alg:tree_construction}, and let $y_1,\ldots,y_G$ be its root-to-leaf trajectories. We optimize
\begin{equation}
\begin{split}
\mathcal J_{\TreeHCA}(\theta)
=\EE_{x\sim\mathcal D,\TT\sim\mathcal B_\pi(\cdot\mid x)}
\left[\frac{1}{G}\sum_{i=1}^G\frac{1}{N_i}\sum_t M_{i,t}\right.\\
\left.\quad\min\!\left\{s_{i,t}(\theta)A_{i,t},
\operatorname{clip}\!\left(s_{i,t}(\theta),1-\varepsilon_l,1+\varepsilon_r\right)A_{i,t}\right\}\right].
\end{split}
\label{eq:treehca_obj}
\end{equation}
Advantages are held fixed during the update. A shared-prefix turn receives the same node advantage in every descendant trajectory, with the sequence ratio evaluated for each complete trajectory.

\section{Experiments}
\label{sec:experiments}

We now conduct a set of experiments to answer two fundamental questions. \emph{Does hindsight-guided exploration collect informative trajectories?} Chain-method comparisons place intermediate credit signals in context, and tree-method comparisons examine the choice of exploration rule. \emph{Does inverse-hindsight weighting provide useful turn-level credit?} The component comparison changes exploration and credit rules in turn, thereby examining their contribution to the complete method.

\subsection{Setup and Comparison Protocols}

We evaluate search-based question answering with Qwen2.5-3B-Instruct against chain and tree baselines (\Cref{tab:validation_results_chain,tab:validation_results_tree}), and with Qwen3-4B and Qwen3-8B (\Cref{tab:main_results}) \citep{Qwen3techreport}. Multi-hop benchmarks comprise HotpotQA, 2WikiMultihopQA, MuSiQue, and Bamboogle \citep{yang2018hotpotqa,2wikimultihop,trivedi2022musique,bamboogle}. Single-hop benchmarks comprise Natural Questions (NQ), TriviaQA, and PopQA \citep{kwiatkowski2019nq,joshi2017triviaqa,popqa}. The agent interleaves reasoning with queries to a Wikipedia index, receiving the top three passages per query. Qwen3 allows six tool calls and 6192 response tokens. \Cref{tab:qwen25_3b_hparams} gives the Qwen2.5 settings. Exact match (EM) is the percentage of questions whose normalized extracted answer matches a reference answer. Aggregate scores weight datasets by test-set size.

For Qwen3, we train on HotpotQA for 250 steps for multi-hop evaluation and NQ for 200 steps for single-hop evaluation. The tree configuration uses $M=10$ initial rollouts, $K=6$ exploration parents, and $L=2$ expansion rounds. We train \ATPO from its public implementation and apply TreeGRPO's likelihood-weighted backup \citep{ding2026treegrpo} to search trajectories. These methods and \TreeHCA share an expansion budget (\Cref{tab:main_results}). Published ReAct \citep{yao2023react}, GRPO, GSPO, and AEPO results complete the comparison. \Cref{app:experimental_details} provides further setup details.

\begin{table}[t]
\caption{Qwen2.5-3B-Instruct comparison for chain-based baselines: EM (\%). Darker to lighter green marks the three largest distinct values per column, with ties sharing a shade and maxima in bold. Underlined entries indicate methods (re-)evaluated in our setup. Each of these methods uses five trajectories.}
\label{tab:validation_results_chain}
\centering
\footnotesize
\setlength\tabcolsep{4.2pt}
\renewcommand\arraystretch{1.02}
\resizebox{\linewidth}{!}{%
\begin{tabular}{llrrrrrrrr}
\toprule
\multirow{2}{*}{\textbf{Method}} & \multirow{2}{*}{\textbf{Type}} & \multicolumn{4}{c}{\textbf{Multi-hop QA}} & \multicolumn{3}{c}{\textbf{Single-hop QA}} & \multirow{2}{*}{\textbf{Avg.}} \\
\cmidrule(r){3-6}\cmidrule(r){7-9}
& & \textbf{Hotpot} & \textbf{2Wiki} & \textbf{MuSiQue} & \textbf{Bamb.} & \textbf{NQ} & \textbf{TriviaQA} & \textbf{PopQA} & \\
\midrule
\rowcolor{tablegroup}\multicolumn{10}{c}{\textbf{Backbone: Qwen2.5-3B}} \\
Search-R1 & Chain-based & 32.4 & 31.9 & 10.3 & 26.4 & 34.1 & 54.5 & 37.8 & 37.4 \\
ZeroSearch & Chain-based & 27.4 & 30.0 & 9.8 & 11.1 & 41.4 & 57.4 & 44.8 & 39.1 \\
\underline{GiGPO} & Chain-based & 36.8 & 36.2 & \rktwo{12.2} & \rkthree{29.6} & 43.1 & \rkthree{60.5} & \rktwo{46.1} & 43.3 \\
\underline{IGPO} & Chain-based & \rkthree{37.5} & \rktwo{36.4} & \rkthree{11.8} & 25.6 & \rktwo{44.9} & \rkone{\textbf{61.8}} & \rkone{\textbf{46.5}} & \rkthree{43.6} \\
HCAPO & Chain-based & \rktwo{38.6} & \rkthree{36.3} & \rkone{\textbf{14.8}} & \rkone{\textbf{64.5}} & \rkthree{44.4} & \rkthree{60.5} & \rkthree{45.5} & \rktwo{44.3} \\
\textbf{\underline{\TreeHCA(n=5)}} & Tree-based & \rkone{\textbf{39.9}} & \rkone{\textbf{39.4}} & \rkone{\textbf{14.8}} & \rktwo{{32.4}} & \rkone{\textbf{45.3}} & \rktwo{61.5} & 44.4 & \rkone{\textbf{44.9}} \\
\bottomrule
\end{tabular}}
\end{table}

\begin{table}[t]
\caption{Qwen2.5-3B-Instruct comparison for tree-based baselines: EM (\%). Colors follow \Cref{tab:validation_results_chain}.}
\label{tab:validation_results_tree}
\centering
\footnotesize
\setlength\tabcolsep{4.2pt}
\renewcommand\arraystretch{1.02}
\resizebox{\linewidth}{!}{%
\begin{tabular}{llrrrrrrrr}
\toprule
\multirow{2}{*}{\textbf{Method}} & \multirow{2}{*}{\textbf{Type}} & \multicolumn{4}{c}{\textbf{Multi-hop QA}} & \multicolumn{3}{c}{\textbf{Single-hop QA}} & \multirow{2}{*}{\textbf{Avg.}} \\
\cmidrule(r){3-6}\cmidrule(r){7-9}
& & \textbf{Hotpot} & \textbf{2Wiki} & \textbf{MuSiQue} & \textbf{Bamb.} & \textbf{NQ} & \textbf{TriviaQA} & \textbf{PopQA} & \\
\midrule
\rowcolor{tablegroup}\multicolumn{10}{c}{\textbf{Backbone: Qwen2.5-3B}} \\
Tree-GRPO & Tree-based
& \rkthree{31.1} & \rkthree{31.7} & \rkthree{11.9} & \rkthree{28.0}
& \rkthree{44.5} & \rktwo{60.4} & \rkthree{43.7} & \rkthree{40.8} \\
\underline{IGRPO} & Tree-based
& \rktwo{38.0} & \rktwo{39.0} & \rktwo{14.3} & \rktwo{29.6}
& \rktwo{45.1} & \rkthree{60.0} & \rkone{\textbf{45.2}} & \rktwo{43.9} \\
\textbf{\underline{\TreeHCA(n=5)}} & Tree-based
& \rkone{\textbf{39.9}} & \rkone{\textbf{39.4}} & \rkone{\textbf{14.8}} & \rkone{\textbf{32.4}}
& \rkone{\textbf{45.3}} & \rkone{\textbf{61.5}} & \rktwo{44.4} & \rkone{\textbf{44.9}} \\
\bottomrule
\end{tabular}}
\end{table}

\subsection{Hindsight-Guided Exploration}
\label{sec:exploration_results}

\Cref{tab:validation_results_chain} compares chain-based methods with \TreeHCA. Search-R1 and ZeroSearch learn from completed reasoning-and-search trajectories \citep{jin2025search,Sun2025ZeroSearch}. GiGPO constructs turn-level comparisons by grouping comparable states, IGPO scores changes in reference-answer likelihood, and HCAPO weights actions using hindsight context \citep{feng2025groupingrouppolicyoptimizationllm,wang2026informationgainbasedpolicyoptimization,tan2026hindsightcreditassignmentlonghorizon}. \Cref{tab:validation_results_tree} adds tree-based baselines. Tree-GRPO uses shared-prefix comparisons \citep{ji2025treesearch}, with its reported scores taken from the reproduction in \citet{IGRPO}. \TreeHCA connects such feedback to the allocation of new continuations, explicitly revisiting histories before low-scoring decisions. It achieves the highest aggregate EM in this comparison, 44.9, exceeding HCAPO by 0.6 percentage points.

\begin{table}[t]
\caption{Qwen3 comparison: EM (\%). ReAct, GRPO, GSPO, and AEPO results are from \citet{zong2026at2poagenticturnbasedpolicy}. TreeGRPO, \ATPO, and \TreeHCA use our shared training protocol. W. Avg. weights datasets by test-set size. Colors rank values within each backbone. 
Underlined entries indicate methods (re-)evaluated in our setup.}
\label{tab:main_results}
\centering
\footnotesize
\setlength\tabcolsep{4.2pt}
\renewcommand\arraystretch{1.02}
\resizebox{\linewidth}{!}{%
\begin{tabular}{lccccccccc}
\toprule
\multirow{2}{*}{\textbf{Method}} & \multicolumn{5}{c}{\textbf{Multi-hop QA}} & \multicolumn{4}{c}{\textbf{Single-hop QA}} \\
\cmidrule(r){2-6}\cmidrule(r){7-10}
& \textbf{Hotpot} & \textbf{2Wiki} & \textbf{MuSiQue} & \textbf{Bamb.} & \textbf{W. Avg.} & \textbf{NQ} & \textbf{TriviaQA} & \textbf{PopQA} & \textbf{W. Avg.} \\
\midrule
\rowcolor{tablegroup}\multicolumn{10}{c}{\textbf{Backbone: Qwen3-4B}} \\
ReAct & 30.42 & 32.92 & 12.83 & 44.80 & 30.01 & 26.75 & 53.53 & 35.34 & 41.31 \\
GRPO & 44.76 & 51.40 & 21.60 & \rkthree{50.40} & 46.02 & 45.98 & \rktwo{65.17} & 49.18 & 54.97 \\
GSPO & \rkthree{47.07} & 49.25 & 22.68 & \rkthree{50.40} & 45.69 & 46.01 & 64.24 & 48.50 & 54.28 \\
\underline{TreeGRPO} & 46.40 & 51.10 & 20.27 & \rktwo{53.60} & 46.26 & \rkthree{46.31} & 64.91 & 49.75 & 55.19 \\
AEPO & 46.36 & \rkthree{51.78} & \rkthree{23.47} & \rkthree{50.40} & \rkthree{46.95} & 45.71 & 64.66 & \rktwo{50.13} & \rkthree{55.20} \\
\underline{\ATPO} & \rktwo{47.44} & \rktwo{52.24} & \rktwo{23.71} & \rkone{\textbf{55.20}} & \rktwo{47.62} & \rktwo{48.34} & \rkthree{64.94} & \rkthree{49.92} & \rktwo{55.55} \\
\textbf{\underline{\TreeHCA}} & \rkone{\textbf{49.23}} & \rkone{\textbf{52.73}} & \rkone{\textbf{24.15}} & \rktwo{53.60} & \rkone{\textbf{48.52}} & \rkone{\textbf{48.78}} & \rkone{\textbf{66.61}} & \rkone{\textbf{50.19}} & \rkone{\textbf{56.38}} \\
\midrule
\rowcolor{tablegroup}\multicolumn{10}{c}{\textbf{Backbone: Qwen3-8B}} \\
ReAct & 20.66 & 19.05 & 9.56 & 37.60 & 18.66 & 21.16 & 41.81 & 27.37 & 32.19 \\
GRPO & 47.01 & \rktwo{53.69} & 21.35 & \rkthree{54.40} & 48.03 & 45.70 & 67.42 & 50.17 & 56.29 \\
GSPO & \rktwo{49.59} & 52.55 & \rktwo{24.35} & \rkthree{54.40} & \rkthree{48.56} & 45.56 & \rkthree{67.75} & 49.66 & 56.15 \\
\underline{{TreeGRPO}} & 47.44 & 52.55 & 21.35 & 50.40 & 47.51 & 47.36 & 67.49 & 50.48 & 56.68 \\
AEPO & \rkthree{49.17} & \rkthree{52.97} & \rkthree{24.01} & \rkthree{54.40} & \rktwo{48.62} & \rktwo{49.92} & \rktwo{68.31} & \rktwo{51.77} & \rktwo{57.94} \\
\underline{\ATPO} & 48.82 & 52.96 & 23.13 & \rktwo{55.20} & 48.41 & \rkthree{49.53} & 66.54 & \rkthree{50.81} & \rkthree{56.75} \\
\textbf{\underline{\TreeHCA}} & \rkone{\textbf{50.51}} & \rkone{\textbf{54.99}} & \rkone{\textbf{24.50}} & \rkone{\textbf{56.00}} & \rkone{\textbf{50.25}} & \rkone{\textbf{51.22}} & \rkone{\textbf{68.49}} & \rkone{\textbf{52.93}} & \rkone{\textbf{58.75}} \\
\bottomrule
\end{tabular}}
\end{table}

\Cref{tab:main_results} examines tree construction and credit on the Qwen3 backbones. Relative to \ATPO, \TreeHCA increases the multi-hop weighted EM from 47.62 to 48.52 at 4B and from 48.41 to 50.25 at 8B. The corresponding single-hop increases are 55.55 to 56.38 and 56.75 to 58.75. Relative to TreeGRPO, the gains are 2.26 and 2.74 percentage points on multi-hop QA, and 1.19 and 2.07 points on single-hop QA. Across all seven methods, \TreeHCA achieves the highest weighted EM for both model sizes and task families. At 8B, it exceeds the strongest published comparison, AEPO, by 1.63 points on multi-hop QA and 0.81 points on single-hop QA.

\subsection{Credit Assignment and the Paired Construction}
\label{sec:component_results}

\begin{table}[t]
\caption{Component comparison on HotpotQA with Qwen3-4B (seed 42). Each row names the exploration rule and the credit rule. EM and F1 validation scores are percentages. Colors follow \Cref{tab:validation_results_chain}.}
\label{tab:ablation}
\centering
\small
\setlength\tabcolsep{7pt}
\begin{tabular}{@{}llrr@{}}
\toprule
\textbf{Exploration} & \textbf{Credit} & \textbf{EM} & \textbf{F1} \\
\midrule
\textbf{Hindsight} & \textbf{\TreeHCA} & \rkone{\textbf{51.76}} & \rkone{\textbf{65.88}} \\
Entropy & \TreeHCA & \rktwo{51.17} & \rktwo{64.17} \\
Entropy & \ATPO & \rkthree{50.78} & \rkthree{63.85} \\
Hindsight & \ATPO & 49.60 & 62.74 \\
Random & \TreeHCA & 49.02 & 61.60 \\
Policy chains & GSPO & 47.46 & 60.76 \\
\bottomrule
\end{tabular}
\end{table}

We vary exploration and credit one at a time (\Cref{tab:ablation}). With the \TreeHCA credit rule, hindsight exploration attains 51.76 EM, compared with 51.17 for entropy and 49.02 for random exploration.

Under hindsight exploration, \TreeHCA credit reaches 51.76 EM and 65.88 F1, versus 49.60 and 62.74 with \ATPO credit. Under entropy exploration, the EM scores are 51.17 and 50.78, respectively. Switching to \TreeHCA credit therefore gains 2.16 EM and 3.14 F1 points under hindsight exploration, and 0.39 EM and 0.32 F1 points under entropy exploration. Under \ATPO credit, however, entropy outperforms hindsight (50.78 versus 49.60 EM). The complete pair is strongest, and both metrics show that the effect of changing exploration depends on the credit rule.

\subsection{Variability and Training Cost}
\label{sec:additional_results}

The accuracy comparisons raise two further questions: \emph{How much does performance vary across training runs?} \emph{What training cost accompanies the gains?} We examine the variation across three seeds and the wall-clock time required by the paired construction.

\paragraph{Variation across training runs.}
\label{app:error_bars}
\Cref{tab:error_bars} reports the multi-hop results across three training seeds. The weighted EM is $48.52\pm0.75$ for Qwen3-4B and $50.25\pm1.27$ for Qwen3-8B. Across individual benchmarks, standard deviations range from 0 to 2.30 points at 4B and from 0.36 to 4.53 points at 8B. The aggregate gains over the strongest baseline are 0.90 and 1.63 points, respectively. Both gains exceed one standard deviation of \TreeHCA's observed run variation, supporting reasonably consistent aggregate performance across these seeds. The larger deviations on some individual benchmarks indicate greater sensitivity to the training seed.

\begin{table}[ht]
\caption{\TreeHCA multi-hop performance as mean $\pm$ one standard deviation over three training seeds. EM and its standard deviation are in percentage points. Means match \Cref{tab:main_results}. W. Avg. weights datasets by size.}
\label{tab:error_bars}
\centering
\footnotesize
\setlength\tabcolsep{5pt}
\renewcommand\arraystretch{1.1}
\begin{tabular}{lccccc}
\toprule
\textbf{Backbone} & \textbf{Hotpot} & \textbf{2Wiki} & \textbf{MuSiQue} & \textbf{Bamb.} & \textbf{W. Avg.} \\
\midrule
Qwen3-4B & $49.23\pm2.30$ & $52.73\pm0.34$ & $24.15\pm1.68$ & $53.60\pm0.00$ & $48.52\pm0.75$ \\
Qwen3-8B & $50.51\pm1.03$ & $54.99\pm1.55$ & $24.50\pm0.36$ & $56.00\pm4.53$ & $50.25\pm1.27$ \\
\bottomrule
\end{tabular}
\end{table}

\paragraph{Training cost.}
The wall-clock training times on eight H200 GPUs are 34.40 versus 31.50 hours at 4B and 36.20 versus 32.58 hours at 8B for \TreeHCA and \ATPO, respectively. These correspond to increases of 9.2\% and 11.1\%. Under the shared expansion budget, the paired design therefore improves accuracy with 9--11\% more training time. The timing comparison quantifies the computational cost accompanying the stronger credit signals.

\section{Conclusion}
\label{sec:conclusion}

\TreeHCA connects exploration and intermediate credit through a common hindsight change of measure. Hindsight scores guide the selection of alternative continuations, and inverse weights account for that sampling preference when estimating advantages. We derive this relationship for a hindsight-tilted proposal and finite-pool candidate selection, then realize it through a normalized tree backup. Across the evaluated question-answering settings, the paired construction improves on the compared baselines, and component ablations connect its performance to the combination of exploration and credit. Together, the sampling identities and component experiments identify the pairing as the central design principle behind these gains.

\clearpage
\section*{Reproducibility Statement}
\Cref{sec:hindsight_guided_branching,sec:credit_assignment} specify the exploration rule and learning objective. Proofs and estimator conditions appear in \cref{app:theory}, and configurations, result provenance, and uncertainty reporting in \cref{app:experimental_details,app:error_bars}. Code: \url{https://github.com/TheBlackCat22/TreeHCA}.

\bibliographystyle{iclr2027_conference}
\bibliography{reference}

\newpage

\appendix
\UseAppendixReferenceNames
\raggedbottom
\makeatletter
\setlength{\@fptop}{0pt}
\setlength{\@fpsep}{12pt}
\makeatother
\section{Detailed related work}
\label{app:extended_related_work}

Intermediate credit assignment depends on both the information available at a decision and the continuations used to assess it. We organize the relevant methods by how they collect alternatives and translate outcomes into learning signals.

\subsection{Hindsight distributions and practical action credit}

Hindsight credit assignment (HCA) measures an action's relevance through its probability conditioned on a later state or return \citep{harutyunyan2019hindsight}. Bayes' rule relates this posterior-to-policy ratio to the action's effect on future outcomes. Under exact conditionals and appropriate support, the forward ratio recovers action-conditioned values from policy-mixture outcomes, while the inverse ratio recovers policy-mixture expectations from action-conditioned outcomes. For binary rewards, the success-conditioned ratio ranks actions by expected reward at a common parent. These identities connect outcome relevance to a change of measure.

GiGPO obtains practical credit by combining episode-relative returns with discounted-return comparisons among actions at recurring anchor states \citep{feng2025groupingrouppolicyoptimizationllm}. It matches environment-state keys, using sequence similarity for search-QA histories. These comparisons can link actions at different times or histories sharing an anchor. HCAPO uses the actor as a generative hindsight critic \citep{tan2026hindsightcreditassignmentlonghorizon}. It scores recorded actions with successful-outcome context, applies length and temperature normalization, divides each score by its trajectory mean, and clips the multiplier before combining discounted-return and trajectory-relative feedback. Both obtain intermediate signals from collected trajectories. \TreeHCA uses the exact hindsight conditional for its ranking interpretation and reference-answer scores to direct the collection of alternatives.

\subsection{Outcome and process feedback along policy trajectories}

Outcome-supervised GRPO assigns each response's standardized reward to its generated tokens, comparing responses to the same prompt \citep{GRPO}. Its process-supervised formulation instead accumulates intermediate reward-model scores. Search-R1 learns interleaved reasoning and retrieval from terminal answer rewards, using PPO with a learned value model by default and GRPO as an alternative \citep{jin2025search}. Retrieved passages condition subsequent decisions, while losses apply to generated tokens. ZeroSearch trains with a retrieval simulator and a curriculum of increasingly noisy observations, combining terminal answer feedback with experience under changing retrieval conditions \citep{Sun2025ZeroSearch}.

Process rewards supply information about intermediate search decisions. StepSearch rewards gains in gold-evidence coverage, penalizes repeated retrieval, and adds answer and query supervision \citep{chen2025stepsearch}. PPO propagates these rewards through value estimates and generalized advantage estimation. IGPO instead uses changes in the reference answer's length-normalized log likelihood as dense turn rewards, normalizing process and outcome feedback separately \citep{wang2026informationgainbasedpolicyoptimization}. These methods refine credit within collected chains. In \TreeHCA, answer information also determines which histories and decisions receive additional rollout budget.

\subsection{Tree structure, exploration criteria, and credit aggregation}

Shared-prefix trees compare continuations from a common history while reusing earlier interaction. The LLM-agent Tree-GRPO of \citet{ji2025treesearch} expands randomly selected nonleaf steps and combines reward-normalized advantages within and across trees. Shared-prefix occurrences accumulate contributions from descendant trajectories. The diffusion-model TreeGRPO of \citet{ding2026treegrpo} samples alternative denoising paths within a stochastic window and propagates normalized terminal advantages through child weights proportional to behavior-policy likelihoods. These distinct methods connect tree structure to credit in agent interactions and latent denoising, respectively.

Entropy provides another exploration criterion. \ATPO ranks histories using sampled-token log probabilities and a penalty for repeated parent expansion, then aggregates leaf signals through entropy-weighted backups and applies a turn-level policy objective \citep{zong2026at2poagenticturnbasedpolicy}. AEPO uses a pilot trajectory to allocate sampling between the root and intermediate histories according to root and tool-related entropy \citep{AEPO}. It spreads exploration through a consecutive-expansion penalty and combines entropy-shaped advantages with a clipping modification that preserves positive-advantage gradients beyond the upper threshold. \TreeHCA connects history and candidate selection to reference-answer compatibility and the corresponding sampling correction.

\subsection{The sampling distribution and policy optimization}

IGRPO directly connects answer information to rollout allocation \citep{IGRPO}. It combines length-normalized answer likelihood with improvement over the parent and exponentially transforms this score into expansion probabilities. Under a fixed policy, equal horizon, and growing per-layer budgets, its analysis gives an exponentially tilted trajectory distribution with a KL-regularized teacher interpretation. Group-relative training uses samples from this teacher. Our estimation target is the original rollout-policy expectation at a selected parent, connecting continuation selection to the reciprocal weighting that recovers that target.

For an independently generated candidate pool with positive selection probabilities, \Cref{eq:finite_pool_correction} gives an exact correction: conditional weighting recovers the pool mean, and averaging over candidate generation recovers the base-policy expectation. The practical tree backup in \Cref{eq:hindsight_corrected_tree_backup} normalizes inverse proxy scores over retained siblings. Under independent draws from the specified direct tilt, the corresponding self-normalized estimator is consistent (\Cref{app:snis}). The sampling law therefore supplies the inverse-score form used in credit aggregation.

Policy optimization addresses a separate distribution change between rollout and updated policies. GSPO uses a length-normalized sequence ratio, while GSPO-token allows token-specific advantages through a stop-gradient construction \citep{GSPO}. In \TreeHCA, this optimizer acts on turn-level advantages after tree credit has been computed. The exploration rule and credit operator share a common principle: credit determines which alternatives are collected, and their collection law informs the credit assigned to their outcomes.

\section{Additional experimental details}
\label{app:experimental_details}

\subsection{Hyperparameters}
\Cref{tab:shared_hparams} groups the Qwen3 settings shared by \ATPO and \TreeHCA according to their role in training. The optimizer and learning rate govern parameter updates, while clipping controls the likelihood-ratio range in the policy objective. The training and PPO mini-batch sizes specify the granularity of these updates. Prompt and response lengths are measured in tokens, and the tool-call limit bounds the number of external interactions along a trajectory.

The rollout settings describe how trajectories are collected. The tree starts with $M$ complete trajectories and adds alternatives from up to $K$ selected parents in each of $L$ expansion rounds. These controls specify the allocation of exploration across histories and rounds, complementing the length and interaction limits on each trajectory.

\begin{table}[!htb]
\centering
\caption{Qwen3 training and rollout settings shared by \ATPO and \TreeHCA.}
\label{tab:shared_hparams}
\begin{tabular*}{\linewidth}{@{\extracolsep{\fill}}lc@{}}
\toprule
\textbf{Hyperparameter} & \textbf{Value} \\
\midrule
\multicolumn{2}{@{}l}{\textit{Policy optimization}} \\
Optimizer & AdamW \\
Learning rate & 1e-6 \\
Clip\_ratio & 0.2 \\
Training batch size & 64 \\
PPO mini-batch size & 8 \\
\midrule
\multicolumn{2}{@{}l}{\textit{Sequence and interaction limits}} \\
Maximum prompt length & 2000 \\
Maximum response length & 6192 \\
Maximum tool calls & 6 \\
\midrule
\multicolumn{2}{@{}l}{\textit{Rollout allocation}} \\
Rollout\_n & 16 \\
Initial rollout count $M$ & 10 \\
Selected parents $K$ & 6 \\
Expansion rounds $L$ & 2 \\
\bottomrule
\end{tabular*}
\end{table}

\begin{table}[!htb]
\centering
\caption{Rollout settings for Qwen2.5-3B experiments.}
\label{tab:qwen25_3b_hparams}
\begin{tabular*}{\linewidth}{@{\extracolsep{\fill}}lc@{}}
\toprule
\textbf{Hyperparameter} & \textbf{Value} \\
\midrule
\multicolumn{2}{@{}l}{\textit{Training}} \\
Number of training steps & 200 \\
Exploration threshold $\kappa_{\mathrm{br}}$ & $-0.14$ \\
\midrule
\multicolumn{2}{@{}l}{\textit{Sequence and interaction limits}} \\
Maximum turns & 4 \\
Maximum rollout length & 512 \\
\midrule
\multicolumn{2}{@{}l}{\textit{Rollout allocation}} \\
GRPO rollout count $n$ & 5 \\
\TreeHCA candidate children $N$ & 2 \\
\bottomrule
\end{tabular*}
\end{table}

\Cref{tab:treehca_hparams} gives the additional controls for hindsight-guided exploration. Score smoothing stabilizes the comparison of reference-answer likelihoods when these likelihoods are small. The negative exploration threshold identifies completed turns that reduce the score, and the sibling penalty lowers the priority of parents that already have alternatives. The candidate count $N$ determines how many completed turns are scored at a selected parent before choosing a continuation.

\begin{table}[!htb]
\centering
\caption{\TreeHCA score and exploration controls.}
\label{tab:treehca_hparams}
\begin{tabular*}{\linewidth}{@{\extracolsep{\fill}}lc@{}}
\toprule
\textbf{Hyperparameter} & \textbf{Value} \\
\midrule
Reference-answer score smoothing $\epsilon_{\mathrm{smooth}}$ & 0.05 \\
Exploration threshold $\kappa_{\mathrm{br}}$ & $-0.14$ \\
Sibling penalty $\gamma$ & 0.05 \\
Candidates per selected parent $N$ & 2 \\
\bottomrule
\end{tabular*}
\end{table}

\subsection{Runtime}
\label{app:runtime}
All training experiments for QA tasks were conducted on a high-performance computing cluster equipped with 8$\times$ NVIDIA H200 GPUs. \Cref{tab:runtime} reports wall-clock training time for Qwen3-4B and Qwen3-8B under \ATPO and \TreeHCA.

\begin{table}[!htb]
\centering
\caption{Wall-clock training time in hours on 8 H200 GPUs.}
\label{tab:runtime}
\begin{tabular*}{\linewidth}{@{\extracolsep{\fill}}lcc@{}}
\toprule
\textbf{Model} & \textbf{\ATPO} & \textbf{\TreeHCA} \\
\midrule
Qwen3-4B & 31.50 & 34.40 \\
Qwen3-8B & 32.58 & 36.20 \\
\bottomrule
\end{tabular*}
\end{table}

\subsection{Code details}
The experiments with the Qwen3 family models use code built upon the codebase for \ATPO \citep{zong2026at2poagenticturnbasedpolicy}, which itself is built on the VeRL framework \citep{Sheng_2025,zhang2024framework}. The experiments with the Qwen2.5-3B-Instruct model use the codebase for IGRPO \citep{IGRPO} with the addition of our \TreeHCA implementation. That codebase uses the VeRL-agent framework \citep{feng2025groupingrouppolicyoptimizationllm}.

\section{Additional component measurements}
\label{app:ablation}

\Cref{tab:ablation} reports EM and F1 for the component comparison. The corresponding validation scores recorded for the six rows are 51.56, 50.00, 49.41, 48.05, 47.66, and 39.26, respectively. These values preserve the distinction between the validation measurement and the EM and F1 metrics used in the main comparison. All six rows use Qwen3-4B on HotpotQA with seed 42.

\begin{figure}[ht]
\centering
\includegraphics[width=0.82\linewidth]{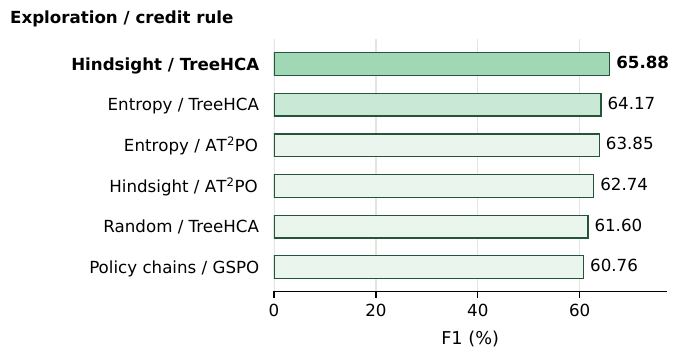}
\caption{F1 for the component comparison in \Cref{tab:ablation}. The figure reports the F1 column, including 65.88 for the complete hindsight-guided exploration and credit pair.}
\label{fig:ablation_em}
\end{figure}
\FloatBarrier

\section{Hindsight Identities and Estimator Conditions}
\label{app:theory}

\subsection{The Return-Conditioned HCA Identity}
\label{app:hca_ranking_proof}

\Cref{lem:hca_counterfactual_actions} instantiates the return-conditioned action-value identity of \citet{harutyunyan2019hindsight}, Theorem~5. The states of our process are complete interaction histories. A composite turn $u=(s,a)$ is an action, the environment supplies its observation, and the undiscounted return equals the terminal verifier reward. This formulation retains the policy's conditioning information and treats the tool response as part of the state transition.

\begin{proof}[Proof of \Cref{lem:hca_counterfactual_actions}]
Fix $h$ and a decision $u$ with $\pi(u\mid h)>0$. For every outcome $r$ with positive probability under a continuation from $h$, Bayes' rule gives
\begin{equation}
\PP_\pi(R=r\mid h)\frac{\PP_\pi(U=u\mid h,R=r)}{\pi(u\mid h)}
=\PP_\pi(R=r\mid h,U=u).
\label{eq:appendix_hca_bayes}
\end{equation}
Multiplying by $r$ and summing over the possible returns proves
\begin{equation}
\EE_{\tau\sim T(h,\pi)}[R(\tau)\rho_{\HCA}(h,u;R(\tau))]
=\sum_r r\PP_\pi(R=r\mid h,U=u)=Q^\pi(h,u).
\end{equation}
Zero-probability return events contribute zero and need no conditional ratio. For a binary terminal reward, only $r=1$ contributes. When $V^\pi(h)>0$, the identity reduces to $Q^\pi(h,u)=V^\pi(h)\rho_{\HCA}(h,u;1)$, proving the shared-parent ranking statement.
\end{proof}

For completed children with $p_\pi(c\mid n)>0$ and $V^\pi(h(n))>0$, apply Bayes' rule to $C=(U,O)$ under the joint kernel $p_\pi(c\mid n)$. The corresponding ratio is
\begin{equation}
\rho_c=\frac{\PP_\pi(C=c\mid h(n),R=1)}{p_\pi(c\mid n)}
=\frac{\PP_\pi(R=1\mid h(n),C=c)}{\PP_\pi(R=1\mid h(n))}.
\end{equation}
Its numerator conditions on the realized observation. The action-conditioned value in the lemma averages over that observation. The practical score in \Cref{eq:practical_log_hindsight_ratio} is evaluated on the completed history, and therefore supplies a proxy for this completed-child quantity.

\subsection{Exact Importance Identities}
\label{app:importance_proof}

\begin{proof}[Proof of \Cref{prop:hindsight_correction}]
For the direct tilt, substitution of $q_s(c)=p(c)s(c)/Z$ yields
\begin{equation}
\int q_s(c)\frac{Z}{s(c)}f(c)\,dc
=\int p(c)f(c)\,dc.
\end{equation}
The same calculation applies to discrete children by replacing integrals with sums. Support coverage ensures that all nonzero target contributions are represented.

For finite-pool selection, condition on the entire candidate pool. Then
\begin{equation}
\sum_{i=1}^N\alpha_i\frac{f(C_i)}{N\alpha_i}
=\frac{1}{N}\sum_{i=1}^N f(C_i).
\end{equation}
Averaging over the iid candidates gives $\EE_p[f(C)]$. The resulting inverse-selection weight is
\begin{equation}
\frac{1}{N\alpha_i}=\frac{\frac{1}{N}\sum_j s(C_j)}{s(C_i)}.
\label{eq:pool_inverse_score}
\end{equation}
For random continuation estimates $\widehat f_i$, the same proof holds if
$\EE[\widehat f_i\mid C_{1:N},I=i]=f(C_i)$. A fresh policy continuation supplies this conditional estimate of the selected child's value.
\end{proof}

The finite-pool identity separates two sources of randomness. Candidate generation gives a sample from the base kernel, and score-based selection changes the conditional distribution over that sample. Its exact correction uses the mean score of all proposed candidates, including candidates whose continuations are not completed.

For the exact success-conditioned ratio, a child with zero success probability has zero proposal mass. The reward-value integral can still satisfy coverage when its integrand vanishes on those children. Recovering a general centered signal instead requires coverage for that signal's support. Positive proxy scores supply full conditional coverage within a finite candidate pool.

\paragraph{Two-child illustration.}
Suppose the rollout policy assigns probability $1/2$ to each of $c_1,c_2$, and their positive scores are $s(c_1)=1$ and $s(c_2)=3$. The tilted proposal samples them with probabilities $1/4$ and $3/4$. Its normalizer is $Z=2$, so the inverse weights are $2$ and $2/3$. For arbitrary child signals $f_1,f_2$, their weighted expectation is
\begin{equation}
\tfrac14(2f_1)+\tfrac34(\tfrac23 f_2)=\tfrac12 f_1+\tfrac12 f_2.
\label{eq:two_child_example}
\end{equation}
The higher-scoring child receives more sampling effort, while each of its observations receives less weight in the base-policy estimate. The same calculation holds conditionally for a candidate pool containing these two children. The score controls the allocation of samples, and the correction preserves the expectation they estimate.

\subsection{The Marginal Law of Finite-Pool Selection}
\label{app:finite_pool}

Let $C_1,\ldots,C_N$ be iid from $p$ and select an index in proportion to a positive score $s$. Exchangeability gives the marginal density of the selected child as
\begin{equation}
q_N(c)=Np(c)\,\EE\!\left[
\frac{s(c)}{s(c)+\sum_{j=2}^N s(C_j)}\right].
\label{eq:finite_pool_marginal}
\end{equation}
To obtain this expression, sum the probability of selecting each of the $N$ positions, condition on the value at that position, and integrate the remaining $N-1$ candidates. In particular, $q_1=p$. For scores bounded above and bounded away from zero, the law of large numbers and dominated convergence give $q_N(c)/p(c)\to s(c)/\EE_p[s(C)]$ wherever $p(c)>0$. This identifies the normalized score tilt as the large-pool limit of the selection rule.

The finite-pool correction in \Cref{eq:pool_inverse_score} is exact at every pool size. The normalized sibling backup in \Cref{eq:hindsight_corrected_tree_backup} instead uses the completed children retained in the tree. That set contains original policy children and alternatives added at adaptively selected parents. Its scores and downstream signals are coupled to the tree-construction history. The direct-tilt model supplies the inverse-score form of the practical operator, and the empirical component comparison measures that operator on the resulting finite trees.

\subsection{Self-Normalized Estimation and Tree Backups}
\label{app:snis}

Suppose $D_1,\ldots,D_m$ are iid draws from the direct tilt $q_s$, with $s>0$ on the base support, $0<Z<\infty$, and $\EE_p|f(C)|<\infty$. The self-normalized estimator is
\begin{equation}
\widehat\mu_m=
\frac{\sum_{j=1}^m f(D_j)/s(D_j)}{\sum_{j=1}^m1/s(D_j)}.
\end{equation}
By the strong law of large numbers, the numerator divided by $m$ converges almost surely to $\EE_p[f]/Z$, and the denominator divided by $m$ converges to $1/Z$. Hence $\widehat\mu_m\to\EE_p[f]$ almost surely. This is a consistency statement for the specified sampling model. At a finite sample size, self-normalization forms a ratio of random sums and generally introduces bias.

The practical tree backup has useful algebraic properties independently of that sampling model. Positive scores produce nonnegative weights summing to one, so the parent signal lies between the smallest and largest child signals. For any constants $a$ and $b$,
\begin{equation}
\sum_c\widetilde w_{n,c}\bigl(a f(c)+b\bigr)
=a\sum_c\widetilde w_{n,c}f(c)+b.
\end{equation}
Consequently, applying the same group-wide affine reward normalization before or after a fixed weighted backup gives the same result. A single-child node propagates its child's signal unchanged. When all leaf rewards agree, \Cref{eq:leaf_advantage} initializes and propagates a zero signal.

The change-of-measure results concern a fixed integrand and a specified selection law. The implemented learning rule combines proxy scores, a finite adaptive tree, group-relative leaf normalization, and a clipped GSPO objective. These components define the algorithm evaluated in \Cref{sec:experiments}. The propositions establish the sampling-estimation relationships that motivate the algorithm.

\section{Tree-construction implementation details}
\label{app:pseudocode}

Two details clarify the construction in \Cref{alg:tree_construction}. First, the practical hindsight score in \Cref{eq:practical_log_hindsight_ratio} uses a \emph{completed} turn $n=(u,o)$. We execute every candidate against the environment before scoring it, then continue only the selected candidate to termination. One exploration step therefore costs $N$ turn executions and a single full rollout. Second, line~\ref{line:score} clips the score at $\kappa_{\mathrm{br}}$, which assigns exactly the value $\kappa_{\mathrm{br}}$ to every node that is not flagged. Taking the $K$ smallest scores with uniform tie-breaking therefore expands the flagged nodes in increasing order of score and fills any remaining slots uniformly at random, which is the selection rule described in \Cref{sec:hindsight_guided_branching}.

For the Qwen2.5-3B experiments, we use a rollout-matched variant of TreeHCA. Rather than first completing a fixed set of initial trajectories and branching afterward, we monitor the hindsight ratio during rollout generation. Whenever the hindsight ratio falls below a predefined threshold, we branch from that decision point and sample two alternative children. This keeps the total rollout budget comparable to GRPO while concentrating additional exploration on decisions identified as potentially consequential by the hindsight signal.

\begin{algorithm}[t]
\caption{Hindsight-guided chain construction}
\label{alg:chain_construction}
\begin{algorithmic}[1]
\REQUIRE Prompt $x$, reference answer $y^\star$, frozen policy $\pi$, candidate count $n$, maximum turns $T_{\max}$, branching threshold $\kappa_{\rm br}$.
\STATE Initialize the trajectory with prompt $x$ and context $\mathcal H_{<0}$.
\FOR{$t=0,\ldots,T_{\max}-1$}
    \STATE Sample a decision $u_t \sim \pi(\cdot\mid\mathcal H_{<t})$ and execute it.
    \STATE Compute its hindsight ratio $\widehat{\rho}_{\HCA}(u_t)$.
    \IF{$\widehat{\rho}_{\HCA}(u_t) < \kappa_{\rm br}$}
        \STATE Sample $n$ candidate decisions
        $\{u_t^{(i)}\}_{i=1}^{n}\sim\pi(\cdot\mid\mathcal H_{<t})$
        and execute each candidate.
        \STATE Compute $\widehat{\rho}_{\HCA}^{(i)}$ for each candidate.
        \STATE Form the hindsight-guided distribution
        \[
        p_i
        =
        \frac{
        \exp\!\left(\log \widehat{\rho}_{\HCA}^{(i)}\right)
        }{
        \sum_{j=1}^{n}
        \exp\!\left(\log \widehat{\rho}_{\HCA}^{(j)}\right)
        }.
        \]
        \STATE Sample $I\sim\mathrm{Categorical}(p_1,\ldots,p_n)$ and set
        $u_t \leftarrow u_t^{(I)}$.
    \ENDIF
    \STATE Append the selected decision $u_t$ and its observation to the trajectory.
    \STATE Update the context $\mathcal H_{<t+1}$.
    \IF{$u_t$ produces a final answer}
        \STATE \textbf{break}
    \ENDIF
\ENDFOR
\STATE \textbf{return} the completed trajectory with its terminal verifier reward.
\end{algorithmic}
\end{algorithm}

\section{Research assets}
\label{app:licenses}

The experiments use the Qwen2.5 and Qwen3 model families and the seven question-answering benchmarks cited in \Cref{sec:experiments}. Dataset annotations supply reference answers for training-time hindsight scoring and evaluation. Retrieved passages enter the agent's history as tool observations and are masked from the policy-token loss. Training uses code from the prior works \ATPO and IGRPO, which are built upon the VeRL and verl-agent frameworks, respectively (\cref{app:experimental_details}).

\section{Sample generations during training}
\label{app:case_study}

The sibling rollouts below come from one tree during Qwen3-8B training on HotpotQA. Their shared prefix retrieves passages mentioning several NBA Finals series (\Cref{shared_prefix_example}). The unsuccessful continuation immediately answers San Antonio Spurs and receives reward $R=0$ (\Cref{fail_example}). The successful continuation searches explicitly for the Miami Heat's 2012 opponent, retrieves Oklahoma City Thunder, and receives $R=1$ (\Cref{success_example}). These trajectories expose alternative decisions under the same initial evidence: answering immediately or seeking further information. \TreeHCA uses their completed outcomes to compute intermediate credit through the tree.

\begin{table}[!htbp]
    \centering
    \fontsize{9pt}{11pt}\selectfont
    \begin{tabular}{p{\dimexpr\linewidth-2\tabcolsep\relax}}
    \midrule
        \rowcolor{gray!20}\textbf{Shared Prefix (identical in both continuations)} \\
    \midrule
        \textbf{Question:} Juwan Howard beat what team to win his first NBA Final?\\
        \textbf{Labeled Answer:} Oklahoma City Thunder \\
    \midrule
        \rowcolor{blue!10}\texttt{<think>} Okay, I need to find out which team Juwan Howard beat to win his first NBA Finals. First, I should confirm who Juwan Howard is. He was an NBA coach, right? Wait, but the question says he beat a team to win his first NBA Final. Hmm, maybe I'm mixing up some details. Let me check. \texttt{</think>} \texttt{<search>} Juwan Howard NBA coach \texttt{</search>} \\
\rowcolor{green!10}Page 1: ``Juwan Howard'' Juwan Antonio Howard (born February 7, 1973) is an American former professional basketball player who is an assistant coach for the Miami Heat of the National Basketball Association (NBA). Howard formerly played for the Heat from 2010 until 2013. A one-time All-Star and one-time All-NBA power forward, he began his NBA career as the fifth overall pick in the 1994 NBA draft, selected by the Washington Bullets. \\
\rowcolor{green!10}Page 2--3: \textbf{*** Omitted for brevity ***} \\
        \rowcolor{blue!10}\texttt{<think>} Wait, the search result says Juwan Howard was a player, not a coach. But the original question says he beat a team to win his first NBA Final. Maybe there's confusion here. Perhaps I should search for ``Juwan Howard NBA Finals'' to see if he was a coach or player in the Finals. \texttt{</think>} \texttt{<search>} Juwan Howard NBA Finals \texttt{</search>} \\
\rowcolor{green!10}Page 1: ``Juwan Howard'' the finals of the NCAA Division I Men's Basketball Championship in 1992 and 1993. Howard won his first NBA championship with Miami in the 2012 NBA Finals and his second NBA championship in the 2013 NBA Finals. \\
\rowcolor{green!10}Page 2: ``Juwan Howard'' of the Fab Five to win an NBA championship, as a role player on the 2011--12 Miami Heat. Howard was on the court as time expired in the series-clinching game because the three stars of the Heat (James, Wade and Bosh) had arranged it. \\
\rowcolor{green!10}Page 3: \textbf{*** Omitted for brevity ***} \\
        \rowcolor{blue!10}\texttt{<think>} The first search result says that Howard won his first NBA championship with Miami in the 2012 NBA Finals, and the second says he was a role player on the 2011--12 Miami Heat. So he was on the Heat team that won the 2012 NBA Finals. But the question is who did he beat. Let me find that information. \texttt{</think>} \texttt{<search>} 2012 NBA Finals losers \texttt{</search>} \\
\rowcolor{green!10}Page 1: ``2012 NBA Finals'' had previously appeared as the SuperSonics in 1996, where they lost to the Michael Jordan-led Chicago Bulls. It also marked the Miami Heat's third appearance in the NBA Finals. The Heat previously appeared in 2006 and 2011, both times against the Dallas Mavericks. It was the first NBA Finals in 14 seasons that was not held in either the states of California or Texas; the 3 teams that won the previous 13 Western Conference titles the Dallas Mavericks, the Los Angeles Lakers and the San Antonio Spurs were eliminated by Oklahoma City in that order. \\
\rowcolor{green!10}Page 2: ``2012 NBA Finals'' straight championship, losing to the Spurs 4--1 in a rematch of the previous season's finals. The Spurs had defeated the Thunder in the Western Conference finals prior to the NBA finals. After the Heat lost the 2014 Finals, LeBron James returned to the Cleveland Cavaliers \ldots \\
\rowcolor{green!10}Page 3: \textbf{*** Omitted for brevity ***} \\
    \midrule
        \textbf{*** Shared parent: the continuations below begin here ***} \\
    \bottomrule
    \end{tabular}
    \caption{Shared prefix of the sibling rollouts on HotpotQA. Blue marks model output and green marks retrieved passages. Tool-response delimiters are omitted.}
    \label{shared_prefix_example}
\end{table}

\begin{table}[!htbp]
    \centering
    \fontsize{9pt}{11pt}\selectfont
    \begin{tabular}{p{\dimexpr\linewidth-2\tabcolsep\relax}}
    \midrule
        \rowcolor{gray!20}\textbf{Unsuccessful Continuation} \hfill Terminal reward $R=0$ \\
    \midrule
        \textbf{*** Shared prefix omitted; see \Cref{shared_prefix_example} ***} \\
        \rowcolor{blue!10}\texttt{<think>} Wait, the first search result says the 2012 NBA Finals were between Miami Heat and the San Antonio Spurs. So, the Heat beat the Spurs. Therefore, the answer is San Antonio Spurs. \texttt{</think>} \\
        \rowcolor{blue!10}Juwan Howard won his first NBA championship with the Miami Heat in the 2012 NBA Finals. The team he beat to achieve this victory was the San Antonio Spurs. \\
        \rowcolor{blue!10}\texttt{<answer>} The final answer is \boxed{\text{San Antonio Spurs}} \texttt{</answer>} \\
    \bottomrule
    \end{tabular}
    \caption{The unsuccessful continuation from the shared parent of \Cref{shared_prefix_example}. It issues no further query and commits to a wrong answer in a single turn.}
    \label{fail_example}
\end{table}

\begin{table}[!htbp]
    \centering
    \fontsize{9pt}{11pt}\selectfont
    \begin{tabular}{p{\dimexpr\linewidth-2\tabcolsep\relax}}
    \midrule
        \rowcolor{gray!20}\textbf{Successful Continuation} \hfill Terminal reward $R=1$ \\
    \midrule
        \textbf{*** Shared prefix omitted; see \Cref{shared_prefix_example} ***} \\
        \rowcolor{blue!10}\texttt{<think>} Hmm, the first search result says that in the 2012 NBA Finals, the Miami Heat were in the series, but it doesn't mention the opponent. The first paragraph says the Heat had previously appeared in 2006 and 2011 against the Mavericks. The Spurs were the team that beat the Heat in the 2014 Finals. So the answer isn't in the search result. Let me think again. Since Howard was on the Heat team that won the 2012 NBA Finals, but the search result is not clear, maybe I should search for ``2012 NBA Finals Miami Heat opponent'' directly. \texttt{</think>} \texttt{<search>} 2012 NBA Finals Miami Heat opponent \texttt{</search>} \\
\rowcolor{green!10}Page 1: ``2012 NBA Finals'' \textbf{*** Identical to Page 1 above; omitted for brevity ***} \\
\rowcolor{green!10}Page 2: ``2012 NBA Finals'' The 2012 NBA Finals was the championship series of the 2011--12 season of the National Basketball Association (NBA), and the conclusion of the season's playoffs. The Eastern Conference champion Miami Heat defeated the Western Conference champion Oklahoma City Thunder 4 games to 1 to win their second NBA title. Heat Small forward LeBron James was named the Finals MVP. \\
\rowcolor{green!10}Page 3: ``2011--12 Miami Heat season'' \ldots On June 21, the Miami Heat won the NBA Championship at home against the Oklahoma City Thunder, \\
        \rowcolor{blue!10}\texttt{<think>} Wait, according to this, the Miami Heat defeated the Oklahoma City Thunder in the 2012 NBA Finals. So, the answer would be the Oklahoma City Thunder. But let me check again. From the last search result, the 2012 NBA Finals were between the Miami Heat and the Oklahoma City Thunder. So, the team Juwan Howard beat was the Oklahoma City Thunder. \texttt{</think>} \\
        \rowcolor{blue!10}Juwan Howard was a player for the Miami Heat team that won the 2012 NBA Finals. The Miami Heat defeated the Oklahoma City Thunder in that series. \\
        \rowcolor{blue!10}\texttt{<answer>} The team Juwan Howard beat to win his first NBA Finals was the \boxed{\text{Oklahoma City Thunder}}. \texttt{</answer>} \\
    \bottomrule
    \end{tabular}
    \caption{The successful continuation from the same shared parent. It declines to answer from the ambiguous passages and reformulates the query, which retrieves an explicit statement of the required fact.}
    \label{success_example}
\end{table}

\FloatBarrier

\end{document}